\documentclass{article}

\usepackage[main,final]{neurips_2026}

\usepackage[utf8]{inputenc}
\usepackage[T1]{fontenc}
\usepackage{microtype}
\usepackage{graphicx}
\usepackage{subcaption}
\usepackage{booktabs}
\usepackage{url}
\usepackage{nicefrac}
\usepackage{xcolor}
\usepackage{caption}

\usepackage{hyperref}

\usepackage{amsmath}
\usepackage{amssymb}
\usepackage{amsthm}
\usepackage{amsfonts}
\usepackage{algorithm}
\usepackage[noend]{algorithmic}
\renewcommand{\algorithmicrequire}{\textbf{Input:}}
\renewcommand{\algorithmicensure}{\textbf{Output:}}
\renewcommand{\algorithmiccomment}[1]{\hfill{\color{gray}$\triangleright$~#1}}
\usepackage{multirow}
\usepackage{tabularx}
\usepackage{colortbl}
\usepackage{tikz}
\usetikzlibrary{arrows.meta,calc,backgrounds,shapes.geometric}
\usetikzlibrary{shapes,arrows,positioning}

\newcommand{\systemname}{\textsc{SAIR}}
\newcommand{\niparagraph}[1]{\smallskip\noindent\textbf{#1}}

\theoremstyle{plain}
\newtheorem{theorem}{Theorem}[section]
\newtheorem{lemma}[theorem]{Lemma}
\newtheorem{proposition}[theorem]{Proposition}
\theoremstyle{definition}

\newtheorem{assumption}[theorem]{Assumption}
\newtheorem{corollary}[theorem]{Corollary}

\title{\systemname: Cost-Efficient Multi-Stage ML Pipeline Autoscaling via In-Context Reinforcement Learning}

\author{%
  Jianchang Su$^{1}$\thanks{Equal contribution. \quad \textsuperscript{$\dagger$}Corresponding author: \texttt{wei.13.zhang@uconn.edu}.} \quad Yifan Zhang$^{1}$\footnotemark[1] \quad Shengkai Lin$^{2}$ \quad Shizhen Zhao$^{2}$ \\
  \textbf{Yusheng Zheng$^{3}$ \quad Yiwei Yang$^{3}$ \quad Wei Zhang$^{1}$\textsuperscript{$\dagger$}} \\
  $^{1}$University of Connecticut \quad $^{2}$Shanghai Jiao Tong University \quad $^{3}$UC Santa Cruz \\
}

\begin{document}

\maketitle

\begin{abstract}
Multi-stage ML inference pipelines are difficult to autoscale because of heterogeneous resources, cross-stage coupling, and dynamic bottleneck migration. Production controllers ignore cross-stage effects, and learning-based autoscalers require offline training data that is unavailable for new pipelines. We present \systemname, an autoscaler for the zero-shot deployment setting that uses an LLM as an in-context reinforcement learning controller, improving its policy online from reward-labeled experience tuples stored in-context. \systemname{} combines Pareto-dominance reward shaping with a provable separation margin between dominated and non-dominated actions, Bayesian-surprise experience retrieval, and a user-space CUDA interception layer that delivers continuous GPU rate control with sub-millisecond actuation overhead. We prove a regret bound that decomposes into retrieval coverage and LLM selection components, each mapping to a concrete \systemname{} mechanism. On four ML serving pipelines under three workload patterns, \systemname{} achieves the best or tied-best P99 latency in every configuration (up to 50\% reduction over the next baseline); billable cost is reduced through replica consolidation, and effective cost (which captures GPU rate-control savings under co-location or MIG) by up to 97\% under GPU sharing (best case; representative 30--60\% over the next-best baseline), validated by co-location measurements on the physical testbed. \systemname{} matches the steady-state performance of a warm-started controller within 2--3 decision intervals and achieves 86\% bottleneck detection accuracy, all from online experience alone.
\end{abstract}

%==============================================================================
\section{Introduction}
%==============================================================================

Modern machine learning serving systems increasingly decompose monolithic inference into multi-stage pipelines to make better use of heterogeneous hardware~\citep{crankshaw2017clipper,crankshaw2020inferline,gunasekaran2022cocktail}. A typical pipeline consists of CPU/memory-bound preprocessing (data transformation), GPU-accelerated inference (neural network execution), and CPU/memory-bound postprocessing (result aggregation), connected through message queues. This architecture improves throughput and modularity, but creates an autoscaling problem that existing controllers cannot handle well.

Three coupled phenomena cause this problem. First, \textit{dynamic bottleneck migration} shifts the performance-limiting stage with workload intensity. Second, \textit{cross-stage coupling} propagates scaling decisions through the pipeline with variable delays: scaling preprocessing alone can trigger a delayed GPU contention spike 1 to 2 intervals later. Third, \textit{temporal causality confounds root-cause identification}: high GPU utilization may indicate either GPU shortage or downstream congestion.

Existing autoscalers fall into three categories, each inadequate for this setting. Rule-based controllers (HPA~\citep{kubernetes2026kubernetes}, VPA) operate on single-stage metrics and lack cross-stage reasoning. Learning-based controllers~\citep{rzadca2020autopilot,qiu2020firm,qiu2023aware} apply deep RL but need offline training data and fail to generalize under distribution shift. Heuristic methods require extensive per-workload tuning. We target a deployment setting that all three handle poorly: \textit{zero-shot} capacity scaling on a new pipeline, starting from an empty trace history under active bottleneck migration. We present \systemname, which uses an LLM as an \textit{in-context reinforcement learning} (ICRL) controller. Knowledge accumulates as $(x, a, r)$ experience tuples in the LLM's context window, so the controller deploys directly and adapts online from observed rewards. Our contributions:

\begin{itemize}
\item An autoscaler for multi-stage pipelines under bottleneck migration. We formulate the control problem as a contextual bandit with an evolving Pareto frontier, and combine three components: Pareto-dominance reward shaping, Bayesian-surprise experience retrieval, and a user-space \texttt{LD\_PRELOAD} CUDA interception layer that gives continuous GPU rate control with sub-millisecond actuation overhead (Figure~\ref{fig:overview}).

\item A regret bound that decomposes into retrieval coverage error $\xi_t$ and LLM selection error $\eta_t$, with a provable constant separation margin for the Pareto reward; each term in the bound corresponds to a concrete \systemname{} component (retrieval, validator, exploration schedule).

\item Empirical validation on four pipelines and three workload patterns: \systemname{} achieves the best or tied-best P99 latency in every configuration with 86\% bottleneck detection accuracy from online experience alone, and ablations confirm each component's contribution. We quantify the cold-start cost: cold-started \systemname{} converges to warm-started performance within 2--3 decisions, while offline-trained FIRM/AWARE policies transfer as non-adaptive policies on real hardware. Co-location measurements on the physical testbed validate the effective-cost model, and a 4-GPU scale experiment confirms the controller generalizes to larger deployments.
\end{itemize}

%==============================================================================
\section{Related works}
\label{sec:related}
%==============================================================================

% We position our work within four research areas.

\niparagraph{Autoscaling for cloud systems.} Autopilot~\citep{rzadca2020autopilot} introduced ML-based vertical scaling at Google but trains its predictor on weeks of historical traces. FIRM~\citep{qiu2020firm} and AWARE~\citep{qiu2023aware} apply deep RL with safety constraints and recover quickly under distribution shift, but both require thousands of offline episodes to bootstrap usable policies. We additionally implement FIRM and AWARE faithfully (native state dimensions, action spaces, and reward functions) and train each to convergence in a queueing-model simulator; the frozen policies transfer as non-adaptive policies on our testbed (Section~\ref{sec:boundary}), a measured instance of the sim-to-real gap. Sinan~\citep{zhang2021sinan} achieves QoS-aware scaling for microservices through learned performance models that need offline profiling per application. HPA~\citep{kubernetes2026kubernetes} remains the production default and lacks cross-stage reasoning. Table~\ref{tab:related} summarizes the design space along four axes that matter for our setting. The key limitation that motivates \systemname{} (an offline data dependency, or a long online warm-up before the controller is safe to deploy) applies equally to more recent gradient-trained autoscalers.

\begin{table}[t]
\centering
\caption{Design space of autoscalers for ML inference pipelines along axes that matter for the zero-shot deployment setting.}
\label{tab:related}
\scriptsize
\setlength{\tabcolsep}{4pt}
\begin{tabular}{lcccc}
\toprule
Approach & Offline data & Cross-stage & GPU rate ctrl & Strongest bound \\
\midrule
HPA / VPA~\citep{kubernetes2026kubernetes} & none & no & no & none \\
Threshold (rule-based) & none & no & no & none \\
FIRM / AWARE~\citep{qiu2020firm,qiu2023aware} & $10^3$ episodes & yes & no & none \\
Sinan~\citep{zhang2021sinan} & per-app profile & partial & no & none \\
InferLine~\citep{crankshaw2020inferline} & per-pipeline profile & yes & no & queueing \\
\rowcolor{gray!15} \systemname{} (this work) & \textbf{none} & \textbf{yes} & \textbf{yes} & \textbf{regret decomp.} \\
\bottomrule
\end{tabular}
\end{table}

\niparagraph{ML inference serving and LLM scheduling.} Clipper~\citep{crankshaw2017clipper} introduced adaptive batching, InferLine~\citep{crankshaw2020inferline} uses queueing-theoretic models, and Cocktail~\citep{gunasekaran2022cocktail} addresses multi-dimensional optimization. InferLine is the closest pipeline-aware predictor to SAIR but requires offline workload characterization, which is unavailable in the zero-shot setting we target. For LLM serving, Orca~\citep{yu2022orca}, vLLM~\citep{kwon2023efficient}, Llumnix~\citep{sun2024llumnix}, QLM~\citep{patke2024queue}, and Chiron~\citep{patke2025hierarchical} focus on scheduling within serving clusters; SAIR addresses multi-stage \emph{pipeline} autoscaling across heterogeneous stages. FastMatch~\citep{su2024fastmatch} addresses similar CPU--GPU imbalance during training.

\niparagraph{In-context learning, RL, and example selection.} LLMs can learn from few-shot examples without parameter updates~\citep{brown2020language,garg2022can,xie2022explanation}. For RL, Algorithm Distillation~\citep{laskin2023context} and supervised pretraining~\citep{lee2023supervised} enable in-context RL. \citet{krishnamurthy2024can} showed that LLMs require explicit interventions to explore reliably in-context; \citet{song2025reward} demonstrated reward feedback alone enables ICRL. Related work also studies federated RL~\citep{lan2025asynchronous}, contextual integrity~\citep{lan2025contextual}, and preference optimization~\citep{lan2025mappo}. SAIR addresses these failure modes via forced probes, positive-only memory, and a Pareto reward margin (Proposition~\ref{prop:pareto}). For example selection, KATE~\citep{liu2022makes} and IDEAL~\citep{zhang2024ideal} optimize selection over a fixed pool; we instead select from a \emph{growing} experience buffer and combine similarity with a Bayesian-surprise term that captures information gain (Appendix~\ref{app:surprisal}).

\niparagraph{Multi-objective optimization and GPU sharing.} NSGA-II~\citep{deb2002fast} introduced Pareto-front approximation; our Pareto reward uses hypervolume contribution~\citep{zitzler1998multiobjective} (see~\citet{hayes2022practical} for an MORL survey). Our retrieval score is a residual analog of Bayesian surprise~\citep{itti2005bayesian}. For GPU sharing, Gandiva~\citep{xiao2018gandiva}, AntMan~\citep{xiao2020antman}, Salus~\citep{yu2020fine}, and Dilu~\citep{lv2025dilu} provide kernel- or driver-level mechanisms; SAIR uses a lighter user-space \texttt{LD\_PRELOAD} interception that throttles kernel launches via token-bucket rate limiting.

%==============================================================================
\section{Problem formulation}
\label{sec:problem}
%==============================================================================

Consider an ML inference pipeline $\mathcal{S} = \{s_1, s_2, \ldots, s_N\}$ with $N$ stages connected through message queues. Each stage $s_i$ processes requests at service rate $\mu_i$ (requests/s) and maintains an input queue with depth $q_i$ (queued requests). The end-to-end latency decomposes as $L_{e2e} = \sum_{i=1}^{N} ( L_i^{proc} + L_i^{queue} )$, where $L_i^{proc}$ is processing latency and $L_i^{queue}$ is queueing delay at stage $i$. Each stage has resource configuration $\mathbf{r}_i = (n_i, c_i, m_i, \rho_i)$ where $n_i \in \mathbb{Z}^+$ is replica count, $c_i \in \mathbb{R}^+$ (CPU millicores) and $m_i \in \mathbb{R}^+$ (memory MB) are CPU and memory allocations, and $\rho_i \in [0,1]$ is the GPU rate ratio (the fraction of physical GPU compute the stage may consume; see Section~\ref{sec:method}). Stage-level utilizations are $u_i^{cpu}, u_i^{gpu} \in [0,1]$, and $L_{p99}$ denotes the 99th-percentile end-to-end latency observed at the round.

We formalize multi-stage autoscaling as a \textit{contextual bandit}: each decision interval $\Delta t$ (default 30\,s) is one round $t$. The abstraction is justified because each stage's queue reaches a quasi-steady-state within $\Delta t$ once $\Delta t$ exceeds the longest stage's relaxation time (settling window 5--10\,s); it becomes invalid when $\Delta t$ is too short (prevented by the cooldowns in Table~\ref{tab:validator}) or under within-interval adversarial shifts (Section~\ref{sec:discussion}). The context is $x_t = (\{ \mathbf{r}_i, q_i, u_i^{cpu}, u_i^{gpu}, L_{p99} \}_{i=1}^{N},\; \mathcal{P}_t)$, where $\mathcal{P}_t$ is the Pareto frontier (Section~\ref{sec:method}); including $\mathcal{P}_t$ in $x_t$ keeps $r_t = R(x_t, a_t)$ a function of the round-$t$ pair alone. Actions are discretized: $\mathcal{A} = \prod_i \mathcal{A}_i$. For CPU stages $\mathcal{A}_i^{cpu} = \{\Delta n_i, \Delta c_i, \Delta m_i\}$ with $\Delta n_i \in \{-1,0,{+}1,{+}2\}$, $\Delta c_i \in \{\pm \gamma_c, 0\}$, $\Delta m_i \in \{\pm \gamma_m, 0\}$; for GPU stages $\mathcal{A}_i^{gpu} = \{\Delta n_i, \Delta \rho_i\}$ with $\Delta \rho_i \in \{-0.1,0,{+}0.1,{+}0.2\}$. Step granularities are $\gamma_c = 500$\,m and $\gamma_m = 256$\,MB (Table~\ref{tab:validator}). We seek a policy $\pi: \mathcal{X} \rightarrow \mathcal{A}$ minimizing cumulative regret.

We make four assumptions that characterize the problem structure:

\begin{assumption}[Bounded Reward Range]
\label{ass:bounded}
The shaped reward satisfies $|R_t| \leq R_{max}$ for all rounds. This holds because each reward component is normalized and clipped to a fixed range (see Algorithm~\ref{alg:reward} in Appendix~\ref{app:reward}).
\end{assumption}

\begin{assumption}[Bottleneck Identifiability]
\label{ass:bottleneck}
There exists a mapping $\phi: \mathcal{X} \rightarrow \{1, \ldots, N\}$ identifying the bottleneck stage such that scaling stage $\phi(x)$ yields expected reward improvement at least $\Delta > 0$, while scaling other stages yields improvement at most $\Delta/2$. This reflects queueing dynamics: in a tandem queue, the stage with highest utilization limits throughput, and scaling it yields the largest marginal gain.
\end{assumption}

\begin{assumption}[Retrieval Coverage Error]
\label{ass:coverage}
Given context $x_t$, the experience selection returns a set $\mathcal{E}_t$ that induces a candidate action set $\mathcal{A}(\mathcal{E}_t) \subseteq \mathcal{A}$. The \emph{coverage gap} is
\begin{equation}
\xi_t = \max\Big(0,\; \mathbb{E}[R(x_t, a^*_t)] - \max_{a \in \mathcal{A}(\mathcal{E}_t)} \mathbb{E}[R(x_t, a)]\Big),
\end{equation}
measuring how much optimality is lost because the retrieved context lacks relevant action patterns.
\end{assumption}

\begin{assumption}[LLM Selection Error on Retrieved Set]
\label{ass:llm}
Conditioned on retrieved experiences $\mathcal{E}_t$, the LLM outputs $a_t \in \mathcal{A}(\mathcal{E}_t)$ such that
\begin{equation}
\max_{a \in \mathcal{A}(\mathcal{E}_t)} \mathbb{E}[R(x_t, a)] - \mathbb{E}[R(x_t, a_t)] \leq \eta_t
\end{equation}
with probability at least $1 - \delta_{LLM}$. This captures the LLM's bounded suboptimality \emph{within} the actions suggested by retrieved experiences. The constraint $a_t \in \mathcal{A}(\mathcal{E}_t) \subseteq \mathcal{A}$ is enforced by design: SAIR presents a discrete action schema via structured JSON output and applies an action validator (Appendix~\ref{app:validator}) that clamps any out-of-range proposal to the nearest feasible action in $\mathcal{A}$.
\end{assumption}

\niparagraph{Scope.} \systemname{} targets coarse-grained capacity scaling under zero-shot deployment with decision intervals of tens of seconds. Three settings are out of scope: (i) per-request scheduling within a serving cluster, addressed by Orca~\citep{yu2022orca}, vLLM~\citep{kwon2023efficient}, Llumnix~\citep{sun2024llumnix}; (ii) deployments with abundant offline traces where FIRM, AWARE, or InferLine can be trained or profiled and will outperform a zero-shot controller; (iii) sub-second control loops, since LLM decision latency is 1--2\,s.

%==============================================================================
\section{\systemname: In-context RL for autoscaling}
\label{sec:method}
%==============================================================================

\begin{figure}[t]
\centering
\input{figures/overview}
\caption{\systemname{} architecture. The control plane runs one in-context RL round per decision interval $\Delta t$ over the multi-stage pipeline in the data plane; circled numbers mark the steps of a round.}
\label{fig:overview}
\end{figure}

\niparagraph{Overview.} Figure~\ref{fig:overview} shows the architecture. In each round, \systemname{} (1) observes the context $x_t$ from stage telemetry, (2) scores the previous action with the Pareto-shaped reward and stores positive episodes in the experience buffer $\mathcal{D}$, (3) retrieves the top-$M$ episodes by Bayesian surprise, (4) queries the LLM with a prompt that combines $x_t$, the retrieved episodes $\mathcal{E}_t$, and the action schema, (5) validates the returned JSON action, and (6) applies replica and CPU/memory changes through the Kubernetes API and GPU rate changes through the CUDA interception layer. Algorithm~\ref{alg:icrl} gives the complete loop, including the forced probes taken with probability $\epsilon_t$.

\niparagraph{In-context RL policy.} We define the policy implicitly through the LLM's conditional distribution:
\begin{equation}
\label{eq:policy}
\pi_{SAIR}(a|x) = P_{LLM}(a \mid x, \mathcal{E}_{1:M}),
\end{equation}
where $\mathcal{E}_{1:M} = \{(x_i, a_i, r_i)\}_{i=1}^{M}$ are past episodes provided in-context. Tuples come from live system interactions; the policy improves as more relevant experiences are retrieved (decreasing $\xi_t$, Theorem~\ref{thm:regret}), and the reward signal determines action selection. Naive ICRL can fail at exploration~\citep{krishnamurthy2024can}, so \systemname{} adds forced $\epsilon_t$-greedy probes, positive-only memory, and the Pareto reward margin below.

\niparagraph{Pareto-based reward shaping.} Threshold-based rewards lead to reactive policies that scale only after SLA violations and oscillate between under- and over-provisioning. We design a multi-objective reward that encourages proactive scaling while balancing latency and cost:
\begin{equation}
\label{eq:reward}
R(x, a) = R_{latency} + R_{cost} + R_{SLA} + R_{proactive} + R_{pareto}.
\end{equation}
The latency component rewards improvement, $R_{latency} = w_L (L_{before} - L_{after})/L_{baseline}$, and the cost component penalizes resource increase, $R_{cost} = -w_C (C_{after} - C_{before})/C_{budget}$. SLA violations grow quadratically:
\begin{equation}
\label{eq:sla}
R_{SLA} = \begin{cases}
-\left(\frac{L_{after}}{T_{SLA}}\right)^2 + 1 & \text{if } L_{after} > T_{SLA} \\
0 & \text{otherwise}
\end{cases}
\end{equation}
% This creates strong pressure to resolve violations: 2$\times$ SLA yields $-3$ penalty, 5$\times$ yields $-24$, incentivizing prompt scaling during violations.

For balancing the fundamental latency-cost tradeoff, we use Pareto-dominance-based reward shaping:
\begin{equation}
\label{eq:pareto}
R_{pareto}(a) = \begin{cases}
1 + H(a) & \text{if } a \in \mathcal{P} \text{ (non-dominated)} \\
\frac{0.8}{1 + d(a, \mathcal{P})} & \text{otherwise}
\end{cases}
\end{equation}
where $\mathcal{P}$ is the current Pareto frontier, $H(a)$ is the hypervolume contribution of $a$, and $d(a,\mathcal{P})$ is the Euclidean distance to the frontier in normalized objective space.

The proactive scaling bonus encourages multi-stage coordinated scaling during violations:
\begin{equation}
\label{eq:proactive}
R_{proactive} = \sigma(x) \cdot \mu(a) \cdot w_{proactive}
\end{equation}
where $\sigma(x) = \max(0, L_{p99}/T_{SLA} - 1)$ is violation severity, $T_{SLA}$ is the per-pipeline SLA latency target, and $\mu(a) = \sum_i |\Delta n_i| + \alpha \sum_i (|\Delta c_i|/\gamma_c + |\Delta m_i|/\gamma_m + |\Delta \rho_i|) + \frac{1}{2}|\text{stages\_scaled}|$ aggregates all action dimensions with $\alpha = 0.5$.

The weight settings ($w_L = 0.7$, $w_C = 0.3$, $w_{proactive} = 0.3$, fully listed in Appendix~\ref{app:hyperparams}) reflect the priority that latency violations are user-visible while resource costs are amortized. The ablation rows ``Equal Wt.'' and ``Cost Focus'' (Table~\ref{tab:ablation}) show that removing or inverting the asymmetry degrades performance; the design relies on the relative ordering of latency over cost. The Pareto reward $R_{pareto}$ remains the primary structural signal (Proposition~\ref{prop:pareto}).

\niparagraph{Bayesian-surprise retrieval.} The LLM's context window restricts how many experiences can be supplied; with potentially thousands of accumulated experiences in $\mathcal{D}$, selecting the most informative $M$ is critical. We select to maximize information gain:
\begin{equation}
\mathcal{E}^* = \arg\max_{\mathcal{E} \subset \mathcal{D}, |\mathcal{E}|=M} I(\mathcal{E}; \pi^* \mid x_{curr})
\end{equation}

We approximate this using a leave-one-out Bayesian-surprise score:
\begin{equation}
\label{eq:surprise}
\text{score}(e) = \text{sim}(x_e, x_{curr}) \cdot \underbrace{|r_e - \mathbb{E}[r \mid \mathcal{D}_{-e}]|}_{\text{surprise}}
\end{equation}
where $\text{sim}(x_e, x_{curr}) = \exp(-\|\tilde{x}_e - \tilde{x}_{curr}\|^2 / 2\sigma^2)$, $\tilde{x}$ is the per-stage $(q_i, u_i^{cpu}, u_i^{gpu})$ vector plus $L_{p99}$ z-score-normalized over the buffer (so distances are unitless), $\sigma{=}1.0$ in this normalized space, and $\mathcal{P}_t$ is summarized by its min-latency/min-cost extreme points appended to $\tilde{x}$. The surprise term identifies unexpected outcomes (Appendix~\ref{app:surprisal}). With diversity regularization:
\begin{equation}
\label{eq:diversity}
\mathcal{E}^* = \arg\max_{\mathcal{E}} \left[ \sum_{e \in \mathcal{E}} \text{score}(e) - \lambda_{div} \sum_{e_i, e_j \in \mathcal{E}} \text{sim}(e_i, e_j) \right]
\end{equation}
The combined objective is submodular but non-monotone, so the $(1-1/e)$ guarantee of \citet{nemhauser1978analysis} fails to apply. We use deterministic greedy selection with $\lambda_{div} = 0.1$, which empirically outperforms both random selection and pure similarity-based retrieval (Table~\ref{tab:ablation}).

\begin{algorithm}[t]
\caption{\systemname{} decision loop}
\label{alg:icrl}
\small
\begin{algorithmic}[1]
\REQUIRE LLM policy $P_{LLM}$, context size $M$, reward threshold $r_{min}$, exploration $(\epsilon_0, \epsilon_{min}, \lambda)$, interval $\Delta t$
\STATE $\mathcal{D} \leftarrow \emptyset$;\ $\mathcal{P} \leftarrow \emptyset$;\ $\epsilon \leftarrow \epsilon_0$ \COMMENT{experience buffer, Pareto frontier}
\FOR{each decision round $t = 1, 2, \ldots$ (every $\Delta t$)}
    \STATE $x_t \leftarrow \textsc{ObserveContext}()$ \COMMENT{per-stage $(\mathbf{r}_i, q_i, u_i^{cpu}, u_i^{gpu})$, $L_{p99}$, and $\mathcal{P}$}
    \STATE $\mathcal{E}_t \leftarrow \textsc{SelectExperiences}(\mathcal{D}, x_t, M)$ \COMMENT{greedy Bayesian-surprise selection, Eqs.~\eqref{eq:surprise}--\eqref{eq:diversity}}
    \IF{$u \sim \mathcal{U}(0,1) < \epsilon$}
        \STATE $a_t \leftarrow \textsc{RandomProbe}()$ \COMMENT{forced exploration}
    \ELSE
        \STATE $a_t \sim P_{LLM}(\cdot \mid x_t, \mathcal{E}_t, \texttt{constraints})$ \COMMENT{in-context policy $\pi_{SAIR}$, Eq.~\eqref{eq:policy}}
    \ENDIF
    \STATE $a_t \leftarrow \textsc{Validate}(a_t)$ \COMMENT{clamp to $\mathcal{A}$; no-op during cooldown (Table~\ref{tab:validator})}
    \STATE $\textsc{Actuate}(a_t)$;\ wait settling window;\ measure $L_t, C_t$ \COMMENT{Kubernetes (CPU), rate-control socket (GPU)}
    \STATE $r_t \leftarrow R(x_t, a_t; \mathcal{P})$ \COMMENT{Pareto-shaped reward, Eq.~\eqref{eq:reward}, Algorithm~\ref{alg:reward}}
    \IF{$r_t > r_{min}$}
        \STATE $\mathcal{D} \leftarrow \mathcal{D} \cup \{(x_t, a_t, r_t)\}$ \COMMENT{positive-only memory}
    \ENDIF
    \STATE $\mathcal{P} \leftarrow \textsc{UpdateFrontier}(\mathcal{P}, (L_t, C_t))$ \COMMENT{keep non-dominated points}
    \STATE $\epsilon \leftarrow \max(\epsilon_{min},\, \lambda\,\epsilon)$ \COMMENT{exploration decay}
\ENDFOR
\end{algorithmic}
\end{algorithm}

\niparagraph{Continuous GPU control via CUDA interception.} We intercept CUDA kernel launches via \texttt{LD\_PRELOAD}, implementing token-bucket throttling with rate ratio $\rho \in [0,1]$:
\begin{equation}
\text{tokens}_{\tau} = T_{max} \cdot \rho
\end{equation}
where each window $\tau$ (10ms) grants tokens proportional to $\rho$. Kernel launches consume tokens proportional to their parallelism (grid blocks); when exhausted, launches block until refill. This achieves three properties critical for effective RL: (1) \textit{Fine-grained action space}: The mechanism supports continuous $\rho$, but SAIR uses bounded discrete increments $\Delta \rho_i \in \{-0.1, 0, +0.1, +0.2\}$ as a safety layer; (2) \textit{Immediate effect}: Rate changes via Unix socket take effect within milliseconds (measured overhead $<$2ms); (3) \textit{Instant reversibility}: SAIR can increase or decrease $\rho$ instantly, unlike replica scaling with asymmetric scale-up/down times.

GPU utilization is normalized to quota: $u_{quota} = \min(1, u_{actual}/\rho)$, ensuring Assumption~\ref{ass:bottleneck} holds regardless of rate limit.

\niparagraph{Positive-only episode filtering.} We store only episodes with positive reward ($r > r_{min}$), spending limited context budget on informative successes; approximately 30\% of decisions yield negative reward and are filtered out. The trade-off between omitting instructive failures and reducing noise is measured by the Store-All ablation (+12.1\% P99; Table~\ref{tab:ablation}).

%==============================================================================
\section{Theoretical analysis}
\label{sec:theory}
%==============================================================================

The pipeline reaches a quasi-steady-state within each decision interval, so we analyze \systemname{} as a contextual bandit. The regret definition applies round-wise regardless of whether contexts are stochastic, adversarial, or history-dependent: at each round $t$ we compare $\pi(x_t)$ to $a^*_t = \arg\max_a R(x_t, a)$ for that specific context. Full proofs and two further results (sample complexity for bottleneck identification, marginal-gain lemma) are in Appendix~\ref{app:proofs}.

\subsection{Regret framework}

\begin{theorem}[SAIR Regret Bound]
\label{thm:regret}
Under Assumptions~\ref{ass:bounded}--\ref{ass:llm}, let $\epsilon_t$ be the exploration probability at round $t$. SAIR achieves:
\begin{equation}
{\small \text{Regret}(T) \!\leq\! \sum_{t=1}^{T}\! (1\!-\!\epsilon_t)(\xi_t \!+\! \eta_t) + \!\left(\sum_{t=1}^{T}\! \epsilon_t + \delta_{LLM} T\right)\! R_{max}}
\end{equation}
where $\xi_t$ is the retrieval coverage gap, $\eta_t$ is the LLM selection error, and $\sum_t \epsilon_t$ accounts for exploration rounds.
\end{theorem}

\textit{Proof sketch.} With probability $\epsilon_t$ the agent explores (at most $R_{max}$ regret); otherwise the LLM policy incurs coverage gap $\xi_t$ plus selection error $\eta_t$ (with $\delta_{LLM}$ failure probability). Summing over $T$ rounds gives the bound. Full proof in Appendix~\ref{app:proofs}.

Theorem~\ref{thm:regret} is a per-round regret decomposition: with constant $\eta_t$, $\delta_{LLM}$, or $\epsilon_{min}$, the bound grows linearly in $T$. Sublinear total regret requires $\sum_t \eta_t$, $\sum_t \delta_{LLM}$, and $\sum_t \epsilon_t = o(T)$; a geometric decay schedule for $\epsilon_t$ contributes $O(1)$, and $\epsilon_{min}$ is annealable. The practical value is that each term maps to a concrete \systemname{} mechanism with an empirical proxy: $\xi_t$ to the No-ICL gap (+17\% P99); $\eta_t$ to the constant Pareto margin (Proposition~\ref{prop:pareto}), tightened by the validator and JSON-constrained output; the exploration cost to $\epsilon_{min} = 0.05$. The ablation in Section~\ref{sec:eval} measures the sign and magnitude of each term.

\subsection{Component bounds}

\begin{proposition}[Pareto Reward: Frontier Separation]
\label{prop:pareto}
Consider the Pareto reward in Eq.~\ref{eq:pareto} with non-dominated branch $\alpha + H(a)$ and dominated branch $\beta/(1 + d(a,\mathcal{P}))$, where $\alpha > \beta > 0$, normalized objectives $\tilde{L}, \tilde{C} \in [0,1]$, and hypervolume taken with respect to reference point $(1,1)$. Any non-dominated action receives reward at least $\alpha$, any dominated action strictly less than $\beta$, and
\begin{equation}
a \in \mathcal{P},\; b \notin \mathcal{P} \quad \Rightarrow \quad R_{pareto}(a) - R_{pareto}(b) > \alpha - \beta.
\end{equation}
With normalized objectives, $H(a) \in [0,1]$, so $R_{pareto}(a) \in [0,\, \alpha + 1]$, satisfying Assumption~\ref{ass:bounded}.
\end{proposition}

\textit{Proof.} Non-dominated actions satisfy $R_{pareto}(a) = \alpha + H(a) \geq \alpha$ since $H(a) \geq 0$. Dominated actions have $d(b, \mathcal{P}) > 0$, so $R_{pareto}(b) = \beta/(1 + d(b, \mathcal{P})) < \beta$. The constant gap $\alpha - \beta$ lets the LLM separate non-dominated from dominated actions given sufficient context. Boundedness follows from $\tilde{L} = L/L_{max}$ and $\tilde{C} = C/C_{max}$. \systemname{} uses $\alpha = 1$ and $\beta = 0.8$ in all experiments (Appendix~\ref{app:hyperparams}); the values are calibrated above LLM reward-perception noise and below the magnitude of the latency/cost reward weights.

\begin{lemma}[Retrieval Coverage Under Smoothness]
\label{lem:coverage}
If $\mathbb{E}[R(x, a)]$ is $\ell$-Lipschitz in $x$ (where $x$ is z-score-normalized per the retrieval kernel) and $\mathcal{E}_t$ contains $(x_e, a_e, r_e)$ with $\|x_e - x_t\| \leq d_t$ and $a_e$ optimal at $x_e$, then $\xi_t \leq 2\ell d_t$ (proof in Appendix~\ref{app:proofs}).
\end{lemma}

\begin{corollary}[Sublinear Cumulative Coverage Regret]
\label{cor:coverage}
Under Lemma~\ref{lem:coverage}, if contexts lie in a bounded $d$-dimensional subspace and $|\mathcal{D}|$ grows by one per round, then $d_t = O(|\mathcal{D}_t|^{-1/d})$ and $\sum_t \xi_t = O(\ell T^{1-1/d})$, sublinear in $T$.
\end{corollary}

The diversity term decreases expected retrieval distance while the similarity term keeps retrievals close to $x_t$, so SAIR's policy improves as the buffer grows. Theorem~\ref{thm:regret} thus accounts for performance: $\xi_t$ is reduced by selection (Corollary~\ref{cor:coverage}); $\eta_t$ by the JSON-constrained output, validator (Table~\ref{tab:validator}), and the Pareto margin (Proposition~\ref{prop:pareto}); the exploration penalty by the decaying schedule. Ablation confirms that asymmetric Pareto weighting is critical (Equal Wt.\ $-3.9\%$, Cost Focus $+27\%$ P99; Table~\ref{tab:ablation}).

%==============================================================================
\section{Experiments}
\label{sec:eval}
%==============================================================================

The 2-GPU testbed is the minimum configuration that exercises bottleneck migration across heterogeneous CPU/GPU stages; cluster-scale efficiency is orthogonal to per-pipeline correctness, since each PipelineAutoscaler is an independent control object. Experiments span four ML serving pipelines (latency 60\,ms--18\,s) on three workload patterns; each processes about 1M requests with 3 seeds (means reported). Inter-seed P99 variation is at most 5\% (Image Classification), 8\% (NLP and Text Generation), and 12\% (Video Analysis), all below cross-baseline gaps reported below.

\subsection{Experimental setup}

\niparagraph{Hardware.} Kubernetes cluster with 2$\times$ NVIDIA RTX A6000 GPUs (49\,GB each), 64 CPU cores, 256\,GB RAM. For scale experiments (Section~\ref{sec:scale}), a second node with 2 additional A6000 GPUs joins as a k3s agent, forming a 4-GPU, 2-node cluster.

\niparagraph{LLM backend and decision cost.} Main results use GPT-5.1-mini (temperature 1.0, DSPy structured output). Each decision sends 4--6K input tokens and at most 300 output tokens; at $\Delta t = 30$\,s this is about 120 calls per hour per pipeline, or \$0.01 to \$0.05 per pipeline-hour at list pricing. This is orders of magnitude below the savings in Table~\ref{tab:raw_cost}, so the LLM call cost is negligible in the cost comparison. We also cross-checked decisions against Claude Haiku and local Qwen-32B on Image Classification: the action proposals after the validator clamp were similar in distribution and produced P99 latencies within 5\% of the GPT-5.1-mini run. The backend only affects $\eta_t$ in Theorem~\ref{thm:regret}.

\niparagraph{Applications.} We evaluate four ML serving pipelines with different computational profiles: \emph{Image Classification} (MobileNetV2 with CPU-heavy preprocessing such as image decoding, resizing, and normalization, GPU inference, and lightweight postprocessing), \emph{NLP Analysis} (BERT-based text analysis with a balanced CPU/GPU workload and dynamic bottleneck migration), \emph{Text Generation} (a Transformer decoder whose GPU-intensive autoregressive generation creates variable latency), and \emph{Video Analysis} (multi-frame processing with high memory requirements and a 10--20\,s latency profile).

\niparagraph{Workload patterns.} Three patterns stress different autoscaling capabilities: \textit{Poisson} (steady-state with random arrivals), \textit{Ramp} (gradually increasing load), and \textit{Burst} (periodic traffic spikes).

\niparagraph{Baselines.} We compare against four production autoscalers: \emph{Static} (one replica per stage), \emph{HPA-CPU} (Kubernetes HPA with a 70\% CPU utilization target and a 60\,s stabilization window), \emph{VPA} (Vertical Pod Autoscaler for resource sizing in \texttt{Auto} update mode), and \emph{Threshold} (rule-based P99 latency thresholds of 100\,ms for CPU and 200\,ms for GPU stages, with a 60\,s cooldown). All baselines use per-application tuned parameters; HPA and Threshold targets were swept over $\{50,60,70,80\}\%$ and $\{50,100,200,500\}$ms respectively, selecting best-performing values. Learning-based controllers (FIRM~\citep{qiu2020firm}, AWARE~\citep{qiu2023aware}) require offline pretraining, making them impractical for zero-shot evaluation; we train faithful implementations to convergence and evaluate their frozen policies in Section~\ref{sec:boundary}.

\niparagraph{Metrics.} (1) P99 end-to-end latency; (2) effective resource cost per 1M requests. We distinguish two cost models. \emph{Billable cost} reflects whole-GPU pricing: $C_{bill} = \sum_{i \in \text{CPU}} n_i c_i p_{cpu} + \sum_{i \in \text{GPU}} n_i p_{gpu}$, where each GPU incurs full cost regardless of utilization. \emph{Effective cost} accounts for GPU rate control: $C_{eff} = \sum_{i \in \text{CPU}} n_i c_i p_{cpu} + \sum_{i \in \text{GPU}} n_i \rho_i p_{gpu}$, where throttling to $\rho_i < 1$ frees capacity for co-located workloads. We report $C_{eff}$ using AWS on-demand pricing (p3.2xlarge GPU, m5.xlarge CPU); this metric is appropriate when GPUs are shared via time-slicing or MIG.

\subsection{Main results}

\begin{figure}[t]
\centering
\includegraphics[width=0.55\linewidth]{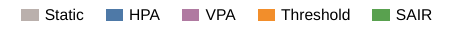}\\[0.2em]
\subfloat[ImgCls]{\includegraphics[width=0.24\linewidth]{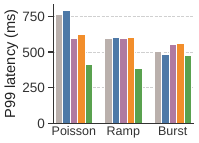}}\hfill
\subfloat[NLP]{\includegraphics[width=0.24\linewidth]{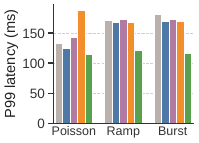}}\hfill
\subfloat[TextGen]{\includegraphics[width=0.24\linewidth]{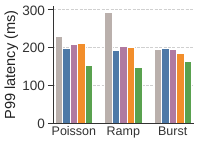}}\hfill
\subfloat[Video]{\includegraphics[width=0.24\linewidth]{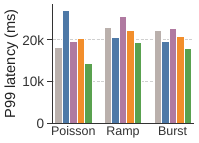}}
\caption{P99 latency across applications and workload patterns. Lower is better; SAIR (rightmost bar in each subplot) wins or ties on every configuration.}
\label{fig:latency_comparison}
\end{figure}

\begin{figure}[t]
\centering
\subfloat[ImgCls]{\includegraphics[width=0.24\linewidth]{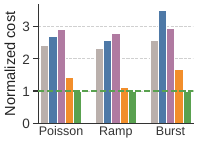}}\hfill
\subfloat[NLP]{\includegraphics[width=0.24\linewidth]{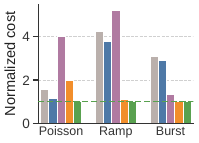}}\hfill
\subfloat[TextGen]{\includegraphics[width=0.24\linewidth]{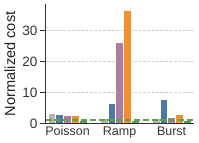}}\hfill
\subfloat[Video]{\includegraphics[width=0.24\linewidth]{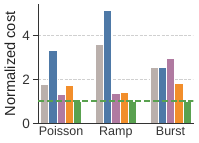}}
\caption{Normalized effective cost (SAIR=1). Values $>$1 indicate higher cost than SAIR. Same legend as Figure~\ref{fig:latency_comparison}.}
\label{fig:cost_comparison}
\end{figure}

SAIR achieves the lowest P99 latency in every configuration (Figure~\ref{fig:latency_comparison}): Image Classification 35--46\%, NLP Analysis 8--31\%, Text Generation 12--50\%, and Video Analysis 6--47\% (the only method that completes the high-latency pipeline within SLA on Burst). The margin over HPA-CPU is largest on workloads where the bottleneck is GPU-bound (NLP, TextGen), which is direct evidence that single-metric reactive controllers produce incorrect scaling under the cross-stage coupling of Section~\ref{sec:problem}. The margin over Threshold is largest on Burst patterns, where Bayesian-surprise retrieval over-represents prior burst episodes and lets the LLM reuse a successful response pattern instead of recomputing it.

SAIR also achieves the best effective cost (Figure~\ref{fig:cost_comparison}): 30--60\% lower than the next-best baseline and up to 81\% lower than HPA on Video Analysis. The Video gain comes primarily from rate-control; in single-tenant mode the gap reduces to the smaller P99-vs-replica margin (Appendix~\ref{app:raw_cost}) but the latency advantage is unaffected.

\niparagraph{Effective-cost validation.} We validated the effective-cost model by running a co-tenant at maximum drive on the same GPU as a throttled primary (Text Generation at 57\,rps). At moderate load, the co-tenant achieved 97.5--98.9\% of standalone throughput at every $\rho$ setting, with the primary's P99 unchanged. Under saturation, the primary's throughput tracked $\rho$ near-proportionally (measured shares 1.00/0.78/0.59/0.38/0.14) while the co-tenant held standalone throughput. Interception overhead: 0.6\%; a second machine replicated the result (98.9--99.6\%, overhead 0.07\%). Full tables in Appendix~\ref{app:colocation}. The 97\% figure is a best case (the co-tenant consumes freed capacity at low primary utilization); representative savings are 30--60\% over the next-best baseline (Table~\ref{tab:raw_cost}).

\subsection{Bottleneck detection and ablation study}

\begin{figure}[t]
\centering
\begin{minipage}[t]{0.48\linewidth}
\centering
\vspace{0pt}
\includegraphics[width=\linewidth]{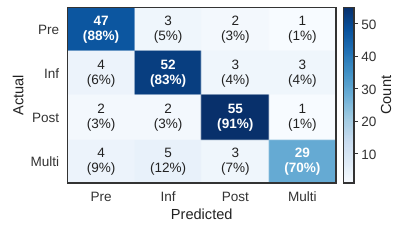}
\caption{Bottleneck detection on 200 injected scenarios. Overall accuracy 86\%; 88--91\% single-stage, 70\% multi-stage. Per-class precision/recall in Appendix~\ref{app:confusion_metrics}.}
\label{fig:confusion}
\end{minipage}\hfill
\begin{minipage}[t]{0.48\linewidth}
\centering
\vspace{0pt}
\includegraphics[width=\linewidth]{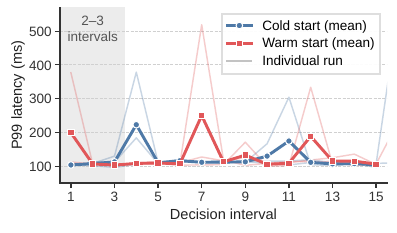}
\caption{Cold-start convergence: per-interval P99 of 3 cold-started and 3 warm-started runs and their means. Cold start reaches warm-start P99 within 2--3 decision intervals.}
\label{fig:coldstart}
\end{minipage}
\end{figure}

\begin{table}[t]
\centering
\caption{Ablation on Text Generation, 3 seeds. $\Delta$T is throughput change, $\Delta$L is P99 latency change (\%), both relative to Full SAIR.}
\label{tab:ablation}
\small
\begin{tabularx}{\linewidth}{l*{4}{>{\centering\arraybackslash}X}}
\toprule
Configuration & Throughput (rps) & P99 (ms) & $\Delta$T (\%) & $\Delta$L (\%) \\
\midrule
\rowcolor{gray!15} Full SAIR & \textbf{84.0} & \textbf{902} & -- & -- \\
No ICL & 82.6 & 1054 & $-$1.7 & +16.8 \\
Store All & 83.0 & 1012 & $-$1.2 & +12.1 \\
Linear Pen. & 80.6 & 962 & $-$4.1 & +6.6 \\
No Bonus & 81.7 & 1006 & $-$2.7 & +11.4 \\
Equal Wt. & 80.7 & 936 & $-$3.9 & +3.7 \\
Cost Focus & 81.4 & 1149 & $-$3.1 & +27.4 \\
Wt.\ Sum & 83.5 & 976 & $-$0.6 & +8.2 \\
Pre Only & 81.9 & 1054 & $-$2.5 & +16.8 \\
\bottomrule
\end{tabularx}
\end{table}

A pre-condition for the regret bound is that the agent can identify which stage to scale. We evaluate bottleneck detection on 200 scenarios with ground-truth bottlenecks injected by throttling individual stages. Figure~\ref{fig:confusion} reports the confusion matrix: 88--91\% accuracy on single-stage bottlenecks enables the targeted scaling that accounts for the main-results latency advantage. Multi-stage accuracy of 70\% (2.8$\times$ the 25\% random baseline) reflects the discrete action space; $R_{proactive}$ compensates by rewarding multi-stage actions during severe violations.

For the ablation (Table~\ref{tab:ablation}), 9 configurations all degrade from Full SAIR. Reward shaping is the largest factor: Linear Pen.\ ($-4.1\%$) performs worst because removing the quadratic shape eliminates the larger penalty after a violation. Equal Wt.\ ($-3.9\%$) and Cost Focus ($-3.1\%$, $+27\%$ P99) show the latency/cost asymmetry matters beyond the magnitudes. Pre-Only ($-2.5\%$) cannot follow the bottleneck once preprocessing scaling shifts the queue downstream.

\niparagraph{Isolating the ICRL contribution.} No-ICL (+16.8\% P99) removes experience retrieval while keeping the same LLM, reward, and validator; Store-All (+12.1\%) keeps retrieval but removes positive-only filtering. These deltas confirm that in-context experience, and its curation, account for the performance gap. A nearest-neighbor replay baseline achieves steady-state P99 comparable to cold-started \systemname{} under stationary load (mean 125.7 vs.\ 136.4\,ms, $n{=}3$); under a 1.7$\times$ workload shift, both configurations show comparable P99, but \systemname{} produces 4--8$\times$ fewer request timeouts in 2 of 3 paired runs, attributable to the LLM's ability to compose novel actions beyond replaying stored ones. The Wt.\ Sum row (+8.2\% P99) isolates the Pareto reward structure by replacing it with a weighted sum of the same component rewards.

\subsection{Cold-start convergence and offline baselines}
\label{sec:boundary}

The zero-shot deployment setting means \systemname{} starts with an empty experience buffer. We quantify the cost of this cold start and the alternative of offline-trained policies.

\niparagraph{Warm vs.\ cold start.} We run 3 interleaved pairs of warm-started (buffer pre-seeded with 13 episodes) and cold-started \systemname{} on Text Generation at 100\,rps, 300\,s per run. Full-run P99: cold 136.4\,ms vs.\ warm 132.7\,ms, a 2.8\% gap within $\pm$15\% inter-run spread. Every cold run reached the warm-started steady-state P99 within 46.6--58.9\,s (2--3 decision intervals; Figure~\ref{fig:coldstart}). Per-run values are in Appendix~\ref{app:warmcold}.

\niparagraph{Offline-trained RL baselines.} We implemented FIRM (DDPG, 10-dim state, 15 actions) and AWARE (PPO, 13-dim state, 7 actions) with their native reward functions, trained each to convergence in an M/M/1 queueing-model simulator per pipeline (seed 42), and evaluated the frozen policies on the live testbed (120\,s $\times$ 3 reps $\times$ \{Poisson, Burst\}). FIRM's DDPG policy never scaled from its initial allocation across all runs. AWARE's PPO policy took sporadic scaling actions under burst load but produced burst-tail P99 up to 556\,ms (2.8$\times$ the 200\,ms SLA). Both policies transferred as non-adaptive policies that held minimal allocations, a measured instance of the sim-to-real gap that zero-shot deployment avoids. Full protocol in Appendix~\ref{app:firmaware}.

\subsection{System scale}
\label{sec:scale}

We grew the system to a 2-node, 4-GPU deployment. A 240\,rps Poisson load that saturates the single-GPU inference stage (P99 27.0\,s) is served at P99 107\,ms with 4 cross-node replicas. \systemname{} controlled this deployment with policy logic, action schema, and prompts unchanged, sustaining full offered load in 3/3 runs at $\geq$99.99\% success. Control-plane concurrency ($N{=}2, 8, 16$) shows per-decision latency flat in $N$ ($N{=}16$: mean 0.64\,s, 416 decisions), cost linear. Per-run details in Appendix~\ref{app:scale}.

%==============================================================================
\section{Discussion and conclusion}
\label{sec:discussion}
%==============================================================================

\niparagraph{Cold start and failure modes.} With an empty buffer, $\epsilon_t$-greedy probes, validator caps, and Bayesian surprise upweighting populate the buffer within $O(N\log N/\Delta^2)$ probes; measured cold-start overhead is 2.8\% of full-run P99, converging in 2--3 decision intervals (Section~\ref{sec:boundary}). Three failure modes remain: (i) stable Poisson load, where the P99 advantage drops below 10\%; (ii) the first 2--3 decisions, dominated by exploration; (iii) multi-stage detection at 70\%, where misidentifications are absorbed by validator caps.

\niparagraph{Prefill/decode serving.} Disaggregated LLM serving~\citep{zhong2024distserve,patel2024splitwise} maps to \systemname{}'s pipeline abstraction (prefill produces KV-cache entries, decode consumes them via a handoff queue). Adapting \systemname{} requires KV-transfer and memory-pressure signals, validation that Assumption~\ref{ass:bottleneck} holds under batching, and benchmarking the CUDA interception layer on newer GPU generations; we leave these extensions to future work.

\niparagraph{Scope, limitations, conclusion.} \systemname{} targets $\ge 3$ heterogeneous stages with bottlenecks migrating within $O(\Delta t)$ absent offline traces. Effective-cost savings of up to 97\% require GPU sharing (validated in Section~\ref{sec:eval}); in single-tenant deployments, savings are 30--60\% over next-best. Cross-cluster scheduling~\citep{su2026trustweave} and $\Delta t < 10$\,s are out of scope. Overall, an LLM controller autoscales multi-stage pipelines from online experience alone, with best-or-tied P99 in every configuration, and the regret framework transfers beyond autoscaling.

\begin{ack}
We thank the anonymous reviewers and the area chair for their feedback. This work is supported in part by the National Science Foundation under Award \#2530909 and by NEUTC.
\end{ack}

\bibliographystyle{plainnat}
\bibliography{references}

\appendix

\section{Proof details}
\label{app:proofs}

\subsection{Proof of Theorem~\ref{thm:regret} (SAIR regret bound)}

\begin{proof}
Let $E_t \in \{0,1\}$ indicate whether round $t$ is a forced exploration round (probability $\epsilon_t$). When $E_t = 1$, the agent takes a random probe and incurs at most $R_{max}$ regret (Assumption~\ref{ass:bounded}).

When $E_t = 0$ (LLM policy), the per-round regret decomposes as:
\begin{align}
&\mathbb{E}[R(x_t, a^*_t)] - \mathbb{E}[R(x_t, a_t)] \\
&= \underbrace{\mathbb{E}[R(x_t, a^*_t)] - \max_{a \in \mathcal{A}(\mathcal{E}_t)} \mathbb{E}[R(x_t, a)]}_{\xi_t \text{ (coverage gap, Assumption~\ref{ass:coverage})}} \\
&\quad + \underbrace{\max_{a \in \mathcal{A}(\mathcal{E}_t)} \mathbb{E}[R(x_t, a)] - \mathbb{E}[R(x_t, a_t)]}_{\leq \eta_t \text{ w.p. } 1 - \delta_{LLM} \text{ (Assumption~\ref{ass:llm})}}
\end{align}

With probability $1 - \delta_{LLM}$, the LLM-policy regret is at most $\xi_t + \eta_t$. With probability $\delta_{LLM}$, the regret is bounded by $R_{max}$.

Combining exploration and exploitation rounds over $T$ rounds:
{\small\begin{align}
\text{Regret}(T) &\leq \textstyle\sum_{t} \big[\epsilon_t R_{max} + (1-\epsilon_t)((\xi_t + \eta_t) + \delta_{LLM} R_{max})\big] \notag \\
&\leq \textstyle\sum_{t} (1-\epsilon_t)(\xi_t + \eta_t) + (\sum_{t} \epsilon_t + \delta_{LLM} T) R_{max}
\end{align}}

The decomposition shows that SAIR's regret depends on three terms: (i) retrieval coverage quality $\xi_t$, reducible through Bayesian-surprise selection; (ii) LLM selection quality $\eta_t$, reducible through prompt constraints and action validation; and (iii) exploration cost $\sum_t \epsilon_t \cdot R_{max}$, controlled by decaying $\epsilon_t$.
\end{proof}

\subsection{Proof of Proposition~\ref{prop:pareto} (frontier separation)}

\begin{proof}
We normalize objectives to $\tilde{L} = L/L_{max} \in [0,1]$ and $\tilde{C} = C/C_{max} \in [0,1]$, and compute hypervolume with respect to reference point $(1,1)$. Let $\alpha > \beta > 0$ be the constants in Eq.~\ref{eq:pareto}.

\textbf{Non-dominated actions ($a \in \mathcal{P}$):} $R_{pareto}(a) = \alpha + H(a)$. Since $H(a) \geq 0$, $R_{pareto}(a) \geq \alpha$. With normalized objectives, $H(a) \leq 1$, so $R_{pareto}(a) \leq \alpha + 1$.

\textbf{Dominated actions ($b \notin \mathcal{P}$):} $R_{pareto}(b) = \beta/(1 + d(b, \mathcal{P}))$. Since $d(b, \mathcal{P}) > 0$, $R_{pareto}(b) < \beta$.

\textbf{Separation margin:} For any $a \in \mathcal{P}$ and $b \notin \mathcal{P}$,
\begin{equation}
R_{pareto}(a) - R_{pareto}(b) > \alpha - \beta.
\end{equation}

\textbf{Boundedness:} $R_{pareto} \in [0,\, \alpha + 1]$, so the total reward has $R_{max} = \alpha + 1$, satisfying Assumption~\ref{ass:bounded}. \systemname{} uses $\alpha = 1$, $\beta = 0.8$ (Appendix~\ref{app:hyperparams}).
\end{proof}

\subsection{Bottleneck detection sample complexity}

\begin{theorem}[Bottleneck Detection Sample Complexity]
\label{thm:bottleneck}
Under Assumption~\ref{ass:bottleneck}, suppose the forced exploration probes in Algorithm~\ref{alg:icrl} generate $m$ scaling experiments per stage under stationary reward distributions. Then SAIR correctly identifies the bottleneck stage with probability at least $1-\delta$ using $m = O(R_{max}^2 N^2 / \Delta^2 \cdot \log(N/\delta))$ probes.
\end{theorem}

The empirical 86\% accuracy in Section~\ref{sec:eval} is consistent with this bound for $N{=}4$.

\begin{proof}
The proof applies to the \emph{forced exploration subsequence} of Algorithm~\ref{alg:icrl}, where the agent takes random probing actions with probability $\epsilon_t$.

\textbf{Step 1: Signal from probes.}
Under Assumption~\ref{ass:bottleneck}, probing the true bottleneck $\phi(x)$ yields reward improvement at least $\Delta$, while probing non-bottleneck stage $j \neq \phi(x)$ yields improvement at most $\Delta/2$. We assume reward distributions are stationary during the probing period.

\textbf{Step 2: Concentration for stage identification.}
After $m$ probes scaling stage $i$, let $\hat{r}_i$ be the empirical mean reward improvement. By Hoeffding's inequality for bounded rewards (Assumption~\ref{ass:bounded}):
\begin{equation}
\mathbb{P}\left[|\hat{r}_i - \mathbb{E}[r_i]| > \frac{\Delta}{4}\right] \leq 2\exp\left(-\frac{m\Delta^2}{8R_{max}^2}\right)
\end{equation}

\textbf{Step 3: Union bound over stage pairs.}
To correctly identify the bottleneck, we need $\hat{r}_{\phi(x)} > \hat{r}_j + \Delta/2$ for all $j \neq \phi(x)$. With concentration $\Delta/4$ for each stage, this holds when true gaps are at least $\Delta/2$.

Union bound over all $\binom{N}{2} \leq N^2/2$ pairs gives failure probability:
\begin{equation}
\mathbb{P}[\text{error}] \leq N^2 \cdot \exp\left(-\frac{m\Delta^2}{8R_{max}^2}\right)
\end{equation}

Setting this to target failure probability $\delta$ and solving for $m$:
\begin{equation}
m = O\left(\frac{R_{max}^2 N^2}{\Delta^2} \log \frac{N}{\delta}\right)
\end{equation}

The total number of decision steps needed is $T_{explore} = m \cdot N / \bar{\epsilon}$ where $\bar{\epsilon}$ is the average exploration rate, since each probe is assigned to a stage uniformly.
\end{proof}

\subsection{Bottleneck marginal-gain lemma}

\begin{lemma}[Bottleneck Marginal Gain]
\label{lem:bottleneck}
In a tandem queueing system with stages having utilizations $u_1, \ldots, u_N$, the stage $i^* = \arg\max_i u_i$ with highest utilization yields the largest marginal latency reduction when scaled.
\end{lemma}

This standard queueing-theoretic result justifies Assumption~\ref{ass:bottleneck} in the main body: pipeline bottlenecks create observable symptoms (high utilization, queue buildup) that distinguish them from non-bottleneck stages.

\subsection{Proof of Lemma~\ref{lem:coverage} (retrieval coverage under smoothness)}

\begin{proof}
Let $a_t^* \in \arg\max_a \mathbb{E}[R(x_t, a)]$. By $\ell$-Lipschitz continuity of $\mathbb{E}[R(\cdot, a)]$ and the optimality of $a_e$ at $x_e$,
$\mathbb{E}[R(x_t, a_t^*)] \leq \mathbb{E}[R(x_e, a_t^*)] + \ell d_t \leq \mathbb{E}[R(x_e, a_e)] + \ell d_t \leq \mathbb{E}[R(x_t, a_e)] + 2\ell d_t.$
Since $a_e \in \mathcal{A}(\mathcal{E}_t)$, $\max_{a \in \mathcal{A}(\mathcal{E}_t)} \mathbb{E}[R(x_t, a)] \geq \mathbb{E}[R(x_t, a_e)]$, so $\xi_t \leq 2\ell d_t$.
\end{proof}

\section{Bayesian-surprise score as information-gain proxy}
\label{app:surprisal}

We justify the leave-one-out Bayesian-surprise score (Eq.~7) as a proxy for information gain under a simple Gaussian model.

\begin{lemma}[Bayesian-Surprise/Information-Gain Connection]
\label{lem:surprisal}
Suppose rewards in a local neighborhood of context space follow $r \sim \mathcal{N}(\mu, \sigma^2)$ with conjugate prior $\mu \sim \mathcal{N}(\mu_0, \tau^2)$. Let $\mathcal{D}_{-e}$ denote the buffer excluding experience $e$, and let $\mu_{-e}$ be the posterior mean given $\mathcal{D}_{-e}$. Then the KL divergence between posteriors (a standard measure of Bayesian surprise) satisfies:
\begin{equation}
D_{KL}\big(p(\mu \mid \mathcal{D}) \,\|\, p(\mu \mid \mathcal{D}_{-e})\big) \propto (r_e - \mu_{-e})^2.
\end{equation}
Therefore, the leave-one-out residual $|r_e - \mathbb{E}[r \mid \mathcal{D}_{-e}]|$ is monotone in information gain under this model.
\end{lemma}

\begin{proof}
Under the conjugate Gaussian model, the posterior given $n$ observations is $\mathcal{N}(\mu_n, \sigma_n^2)$ where $\mu_n = (\tau^{-2}\mu_0 + \sigma^{-2}\sum r_i)/(\tau^{-2} + n\sigma^{-2})$. The KL divergence between two Gaussians is $D_{KL} = \frac{1}{2}[(\mu_n - \mu_{n-1})^2/\sigma_{n}^2 + \sigma_{n-1}^2/\sigma_n^2 - 1 + \log(\sigma_n^2/\sigma_{n-1}^2)]$. Since $\sigma_n^2$ and $\sigma_{n-1}^2$ depend only on $n$ (not on observations), the only observation-dependent term is $(\mu_n - \mu_{n-1})^2 \propto (r_e - \mu_{-e})^2$.
\end{proof}

This provides theoretical grounding for the score: experiences with large leave-one-out residuals correspond to high information gain about the local reward structure, making them most useful for in-context RL.

\section{System implementation details}
\label{app:system}

\niparagraph{Kubernetes operator.} \systemname{} is implemented as a Kubernetes operator using the Kopf framework. The operator watches PipelineAutoscaler custom resources and manages the SAIR agent lifecycle. The agent runs as a sidecar container alongside the pipeline stages.

\niparagraph{Metrics collection.} Context observation uses three sources: (1) GPU metrics via the \texttt{LD\_PRELOAD} library that intercepts CUDA calls and reports active/idle time to a Unix socket; (2) CPU metrics via cgroup accounting from the Kubernetes metrics API; (3) queue depth via the monitor endpoint.

\niparagraph{Action execution.} CPU stage scaling uses Kubernetes Deployment API with rolling updates. GPU rate control uses direct Unix socket communication with the rate-limiting library, bypassing Kubernetes for millisecond response time.

\niparagraph{LLM integration.} \systemname{} can use any model such as GPT-5.1-mini, Claude Haiku, or local Qwen-32B with temperature 1.0 (required for reasoning models) and structured output via DSPy. The prompt includes: current context (metrics), selected experiences (5--15 episodes), and constraints (resource limits, cost budget). The LLM returns a structured JSON action that is parsed and validated before execution.

\section{Action validator and constraints}
\label{app:validator}

Before executing any LLM-proposed action, SAIR applies the following validation rules to ensure safety and feasibility:

\begin{table}[h]
\centering
\caption{Action validation constraints applied to every LLM decision.}
\label{tab:validator}
\small
\begin{tabular}{lc}
\toprule
Constraint & Value \\
\midrule
Max replica change $|\Delta n_i|$ per step (all) & $\leq 2$ \\
Max CPU change $|\Delta c_i|$ per step, $\gamma_c$ (CPU) & 500\,m \\
Max memory change $|\Delta m_i|$ per step, $\gamma_m$ (CPU) & 256\,MB \\
Max GPU rate change $|\Delta \rho_i|$ per step (GPU) & $\leq 0.2$ \\
Min GPU rate ratio $\rho_{min}$ (GPU) & 0.1 \\
Max replicas per stage $n_{max}$ & 8 \\
Cooldown between scale-ups & 60\,s \\
Cooldown between scale-downs & 120\,s \\
Min replicas per stage $n_{min}$ & 1 \\
\bottomrule
\end{tabular}
\end{table}

If the LLM proposes an action violating any constraint, the validator clamps the action to the nearest feasible value. For example, a proposed $\Delta n = +3$ is clamped to $+2$. Actions during cooldown periods are set to no-op for the affected stage.

\section{Reward computation pseudocode}
\label{app:reward}

\begin{algorithm}[h]
\caption{Pareto-shaped reward $R(x, a; \mathcal{P})$}
\label{alg:reward}
\small
\begin{algorithmic}[1]
\REQUIRE context $x$, action $a$, frontier $\mathcal{P}$, SLA target $T_{SLA}$, weights $w_L{=}0.7$, $w_C{=}0.3$, $w_{proactive}{=}0.3$
\ENSURE shaped reward $R \in [-R_{max}, R_{max}]$
\STATE $(L_{before}, C_{before}) \leftarrow$ P99 latency and cost before the action
\STATE $(L_{after}, C_{after}) \leftarrow$ P99 latency and cost after the settling window
\STATE $R_{latency} \leftarrow w_L (L_{before} - L_{after}) / L_{baseline}$ \COMMENT{latency improvement}
\STATE $R_{cost} \leftarrow -w_C (C_{after} - C_{before}) / C_{budget}$ \COMMENT{resource-increase penalty}
\STATE $R_{SLA} \leftarrow 1 - (L_{after}/T_{SLA})^2$ if $L_{after} > T_{SLA}$, else $0$ \COMMENT{quadratic SLA penalty, Eq.~\eqref{eq:sla}}
\STATE $\sigma \leftarrow \max(0,\, L_{p99}/T_{SLA} - 1)$ \COMMENT{violation severity}
\STATE $\mu \leftarrow \sum_i |\Delta n_i| + 0.5 \sum_i (|\Delta c_i|/\gamma_c + |\Delta m_i|/\gamma_m + |\Delta \rho_i|) + \frac{1}{2}|\text{stages\_scaled}|$ \COMMENT{action magnitude}
\STATE $R_{proactive} \leftarrow \sigma \cdot \mu \cdot w_{proactive}$ \COMMENT{proactive bonus, Eq.~\eqref{eq:proactive}}
\STATE $\tilde{L} \leftarrow L_{after}/L_{max}$;\ $\tilde{C} \leftarrow C_{after}/C_{max}$ \COMMENT{normalize to $[0,1]$}
\IF{$(\tilde{L}, \tilde{C})$ is non-dominated in $\mathcal{P}$}
    \STATE $R_{pareto} \leftarrow 1 + H(a)$ \COMMENT{hypervolume contribution w.r.t.\ $(1,1)$, Eq.~\eqref{eq:pareto}}
\ELSE
    \STATE $R_{pareto} \leftarrow 0.8 / (1 + d(a, \mathcal{P}))$ \COMMENT{Euclidean distance to $\mathcal{P}$, Eq.~\eqref{eq:pareto}}
\ENDIF
\STATE $R \leftarrow \mathrm{clip}(R_{latency} + R_{cost} + R_{SLA} + R_{proactive} + R_{pareto},\, -R_{max},\, R_{max})$ \COMMENT{Eq.~\eqref{eq:reward}}
\STATE \textbf{return} $R$
\end{algorithmic}
\end{algorithm}

\section{Hyperparameters}
\label{app:hyperparams}

\begin{table}[h]
\centering
\caption{Hyperparameter settings for all experiments.}
\small
\begin{tabular}{lc}
\toprule
Parameter & Value \\
\midrule
Context window size $M$ & 15 episodes \\
Reward threshold $r_{min}$ & 0 \\
Latency reward weight $w_L$ & 0.7 \\
Cost reward weight $w_C$ & 0.3 \\
Proactive bonus weight $w_{proactive}$ & 0.3 \\
Pareto reward non-dominated anchor $\alpha$ & 1.0 \\
Pareto reward dominated anchor $\beta$ & 0.8 \\
Diversity regularization $\lambda_{div}$ & 0.1 \\
Gaussian kernel bandwidth $\sigma$ (z-score space) & 1.0 \\
Initial exploration rate $\epsilon_0$ & 0.15 \\
Exploration decay $\lambda$ & 0.95 \\
Minimum exploration $\epsilon_{min}$ & 0.05 \\
Decision interval $\Delta t$ & 30 seconds \\
Settling window & 5--10 seconds \\
LLM backend (main results) & GPT-5.1-mini \\
LLM temperature & 1.0 \\
\bottomrule
\end{tabular}
\end{table}

\section{Raw cost data}
\label{app:raw_cost}

Table~\ref{tab:raw_cost} presents the raw cost per 1000 requests for all experiments.

\begin{table}[h]
\centering
\caption{Raw cost per 1K requests (\$, AWS on-demand: p3.2xlarge GPU, m5.xlarge CPU).}
\label{tab:raw_cost}
\small
\begin{tabular}{@{}llccccc@{}}
\toprule
App & Pattern & Static & HPA & VPA & Thresh. & SAIR \\
\midrule
ImgCls & Poisson & .058 & .065 & .070 & .034 & \textbf{.024} \\
 & Ramp & .168 & .185 & .202 & .081 & \textbf{.072} \\
 & Burst & .195 & .267 & .223 & .127 & \textbf{.076} \\
\midrule
NLP & Poisson & .011 & .008 & .028 & .014 & \textbf{.007} \\
 & Ramp & .038 & .034 & .047 & .010 & \textbf{.009} \\
 & Burst & .028 & .026 & .012 & \textbf{.009} & \textbf{.009} \\
\midrule
TextGen & Poisson & .018 & .017 & .014 & .014 & \textbf{.006} \\
 & Ramp & .012 & .047 & .188 & .263 & \textbf{.007} \\
 & Burst & .012 & .053 & .013 & .020 & \textbf{.007} \\
\midrule
Video & Poisson & .928 & 1.764 & .704 & .916 & \textbf{.530} \\
 & Ramp & 1.764 & 2.520 & .657 & .691 & \textbf{.490} \\
 & Burst & 1.411 & 1.411 & 1.642 & 1.010 & \textbf{.557} \\
\bottomrule
\end{tabular}
\end{table}

\section{Bottleneck detection metrics}
\label{app:confusion_metrics}

Table~\ref{tab:confusion_metrics} presents the per-class precision, recall, and F1-score for bottleneck detection.

\begin{table}[h]
\centering
\caption{Per-class precision and recall for bottleneck detection.}
\label{tab:confusion_metrics}
\small
\begin{tabular}{lccc}
\toprule
Class & Precision & Recall & F1-Score \\
\midrule
Preprocessing & 87.0\% & 88.7\% & 87.8\% \\
Inference & 83.9\% & 89.7\% & 86.7\% \\
Postprocessing & 89.8\% & 91.7\% & 90.7\% \\
Multiple & 82.9\% & 70.7\% & 76.3\% \\
\midrule
Overall accuracy & \multicolumn{3}{c}{86.0\%} \\
\bottomrule
\end{tabular}
\end{table}

\section{Prompt template}
\label{app:prompt}

The SAIR agent uses a structured prompt with three components, implemented via DSPy~\citep{khattab2024dspy} signatures. We reproduce the complete template below for reproducibility.

\niparagraph{Input fields.}
\begin{itemize}
\item \texttt{sampled\_episodes}: Past episodes formatted as learning history with reward labels. Each episode shows: context $\rightarrow$ action $\rightarrow$ reward.
\item \texttt{current\_input}: Current pipeline metrics formatted as stage-level statistics (replicas, queue depth, CPU/GPU utilization, allocated resources, end-to-end latency).
\item \texttt{constraints}: Physical infrastructure limits and scaling guidelines (see below).
\end{itemize}

\niparagraph{Output schema.} The LLM returns a structured JSON action:
{\small\begin{verbatim}
{"preprocessing": {
   "action": "scale_replicas"
     |"scale_resources"|"scale_both"
     |"none",
   "replicas": <int>,
   "cpu_millicores": <int>,
   "memory_mb": <int>},
 "inference": {
   "action": "adjust_rate"|"none",
   "rate_ratio": <float in [0,1]>},
 "postprocessing": {
   // same as preprocessing}}
\end{verbatim}}

\niparagraph{Absolute-to-delta mapping.} The LLM outputs \emph{absolute} target values (e.g., \texttt{replicas: 3}, \texttt{cpu\_millicores: 1000}). The executor converts these to deltas relative to the current state, then the action validator (Table~\ref{tab:validator}) clamps each delta to the per-step discretization grid $\mathcal{A}_i$ before execution. This ensures Assumption~\ref{ass:llm} holds: regardless of the LLM's raw output, the executed action is always in $\mathcal{A}$.

\niparagraph{Constraints prompt (excerpt).}
\begin{quote}
\small
\textit{Physical Constraints: Max GPUs: 2, Max CPU cores: 64, Max cost budget: \$100/hour.}

\textit{Scaling Guidelines: Current resource allocations are shown in state. Aggressive reduction ($>$50\%) is dangerous. If p99 $>$ SLA and utilization is low, a burst already passed but the system failed to handle it; needs more capacity. Trust the latency: p99 $>$ SLA means the system failed.}

\textit{Goals: Primary: Keep p99 latency below SLA target. Secondary: Minimize cost. Learn from outcomes: if scaling a stage didn't improve latency, avoid scaling it next time.}
\end{quote}

\niparagraph{Episode format.} Each in-context episode is formatted as:
\begin{verbatim}
Episode (Reward: +0.85 GOOD):
Input:
  preprocessing: replicas=2, queue=15
    CPU_usage=78.3%, ...
  inference: replicas=1, queue=3
    GPU_usage=45.2%, rate_ratio=0.80
  E2E Latency: p99=342ms
Prediction:
  {"preprocessing": {"action":
   "scale_replicas", "replicas": 3},
   "inference": {"action": "none"}}
Reward: +0.85
\end{verbatim}

Episodes are ordered by reward (curriculum ordering) so the LLM observes a learning progression from lower to higher rewards.

\section{Warm/cold start protocol and per-run data}
\label{app:warmcold}

We ran three interleaved pairs (C1, W1, C2, W2, C3, W3) on Text Generation at 100\,rps steady Poisson, 300\,s per run, 15\,s decision intervals. The warm-started controller was pre-seeded with the 13 positive episodes from cold run~1. LLM backend: GPT-5.2 (zero API failures across ${\sim}$90 decisions). Convergence criterion (fixed before inspection): the first time the cold-started controller's rolling-window P99 enters a 5\% band around the warm-started controller's steady-state mean and stays for two consecutive windows.

\begin{table}[h]
\centering
\caption{Warm vs.\ cold start per-run results.}
\label{tab:warmcold}
\small
\begin{tabular}{@{}lccc@{}}
\toprule
Run & Full-run P99 (ms) & Convergence time (s) & Paired diff (ms) \\
\midrule
Cold 1 & 146.1 & 46.6 & +20.6 \\
Warm 1 & 125.5 & -- & \\
Cold 2 & 112.1 & 54.4 & $-$21.2 \\
Warm 2 & 133.3 & -- & \\
Cold 3 & 151.1 & 58.9 & +11.7 \\
Warm 3 & 139.4 & -- & \\
\midrule
Mean cold & 136.4 & 53.3 & +3.7 \\
Mean warm & 132.7 & -- & \\
\bottomrule
\end{tabular}
\end{table}

The paired differences vary in sign across runs. The early-60s window (first 4 decisions) shows cold P99 of 115.0\,ms (mean) vs.\ warm 122.1\,ms; the cold start incurs zero penalty under this load because the initial provisioning is adequate and the zero-shot prior acts sensibly from decision~1.

\section{FIRM/AWARE training and frozen evaluation}
\label{app:firmaware}

We implemented FIRM (DDPG) and AWARE (PPO) as adapter modules with native state dimensions (FIRM: 10, AWARE: 13), native action spaces (FIRM: 15 actions including 6 no-op I/O and network actions with $-0.25$ penalty; AWARE: 7 actions), and native reward functions. Both share the same environment wrapper and the evaluation harness computes all external KPIs (latency, SLA compliance, cost) uniformly using the same cost constants.

Training was conducted in an M/M/1 queueing-model simulator (analytic per-stage latency, randomized per-episode load and initial provisioning). Both algorithms learned and each roughly halved SLA violations vs.\ an untrained network (${\sim}$40\%$\to$${\sim}$21\% on Text Generation). Both over-provision at idle, consistent with their reward functions (both reward functions omit a cost term).

Frozen evaluation used the live testbed: we loaded checkpoints via the harness eval path, 120\,s $\times$ 3 reps $\times$ \{Poisson, Burst\}, with the harness's per-app load profile.

\begin{table}[h]
\centering
\caption{Frozen RL baseline evaluation (Text Generation, live testbed).}
\label{tab:firmaware}
\small
\begin{tabular}{@{}llccc@{}}
\toprule
Policy & Workload & Avg P99 (ms) & Max P99 (ms) & Behavior \\
\midrule
FIRM-DDPG & Poisson & 61.7 / 62.2 / 63.3 & 66 & never scales \\
FIRM-DDPG & Burst & 163.1 / 163.9 / 166.2 & 187 & never scales \\
AWARE-PPO & Poisson & 60.9 / 61.3 / 66.0 & 88 & never scales \\
AWARE-PPO & Burst & 166.5 / 236.3 / 280.0 & 556 & sporadic \\
\bottomrule
\end{tabular}
\end{table}

FIRM's DDPG policy converged to a no-op in all conditions. AWARE's PPO policy attempted sporadic scaling under burst (preprocessing replicas up to 2.6 avg) but produced burst-tail P99 up to 2.8$\times$ the 200\,ms SLA. The 10\,s decision interval (baselines) vs.\ 15\,s (\systemname{}) favors the baselines. Training seed: 42.

\section{Co-location measurements}
\label{app:colocation}

Co-location experiments measure whether the GPU capacity that effective-cost accounting treats as freed is physically consumable. The setup uses a Text Generation (GPT-2) inference container with \texttt{LD\_PRELOAD} token-bucket interception on a dedicated GPU. The primary runs at 57\,rps (30\% of 190\,rps maximum); the co-tenant is a second unthrottled instance on the same GPU at closed-loop maximum drive.

\begin{table}[h]
\centering
\caption{Standard condition: co-tenant share of standalone throughput.}
\label{tab:coloc_standard}
\small
\begin{tabular}{@{}cccc@{}}
\toprule
$\rho$ & Co-tenant (rps) & Share of standalone & Primary P99 (ms) \\
\midrule
1.0 & 188.0 & 98.9\% & 67.7 \\
0.7 & 187.9 & 98.8\% & 68.9 \\
0.5 & 185.4 & 97.5\% & 69.5 \\
0.3 & 185.5 & 97.5\% & 68.6 \\
0.1 & 185.7 & 97.7\% & 69.1 \\
\bottomrule
\end{tabular}
\end{table}

Standalone primary P99 was 67.3\,ms ($\rho{=}0.5$) and 67.9\,ms ($\rho{=}0.1$), indistinguishable from shared-GPU values.

For the calibrated saturation condition, we walked the token cap down until throughput fell below $0.5\times$ baseline, yielding a knee at $3{\times}10^{-4}$; the calibrated cap is $\rho^* = 6{\times}10^{-4}$.

\begin{table}[h]
\centering
\caption{Calibrated saturation: proportional GPU partitioning under contention.}
\label{tab:coloc_calibrated}
\small
\begin{tabular}{@{}ccccc@{}}
\toprule
$\rho_{rel}$ & Primary (rps) & Share of $\rho_{rel}{=}1.0$ & Co-tenant (rps) & Primary P99 (ms) \\
\midrule
1.0 & 99.1 & 1.00 & 180.1 & 496 \\
0.7 & 76.9 & 0.78 & 181.7 & 616 \\
0.5 & 58.4 & 0.59 & 181.5 & 743 \\
0.3 & 38.1 & 0.38 & 182.1 & 1211 \\
0.1 & 14.0 & 0.14 & 184.1 & 2378 \\
\bottomrule
\end{tabular}
\end{table}

Primary throughput tracks $\rho$ near-proportionally (0.14--1.00 measured vs.\ 0.1--1.0 nominal). The co-tenant is insulated (180--184\,rps vs.\ 181.8 standalone). On a second machine, co-tenant shares are 98.9--99.6\% (GPU~0) and 99.0--99.1\% (GPU~1), overhead 0.07\%, with the same proportional shares under calibrated contention.

\newpage
\section{System scale experiments}
\label{app:scale}

\niparagraph{4-GPU deployment.} The second node (cloudsys02) joined the k3s cluster as an agent. The Text Generation inference stage scaled across both machines behind one Service (kube-proxy distributes traffic cross-node). Fixed offered load: 240\,rps Poisson, 180\,s per phase, preprocessing 6 replicas, postprocessing 4 replicas pinned.

\begin{table}[h]
\centering
\caption{Inference scaling curve at a fixed 240\,rps offered load; one replica saturates.}
\label{tab:scale_curve}
\small
\begin{tabular}{@{}lccc@{}}
\toprule
Inf.\ replicas (placement) & Served (rps) & P50 (ms) & P99 (ms) \\
\midrule
1 (node~1) & 237.9 & 13{,}379 & 27{,}006 \\
2 (1+1 cross-node) & 240.1 & 72.0 & 120.2 \\
4 (2+2 cross-node) & 240.8 & 69.9 & 107.3 \\
\bottomrule
\end{tabular}
\end{table}

\systemname{} controlled the 4-GPU deployment (max inference replicas set to 4, 2+2 cross-node) in 3 runs: served 247.5/234.0/238.4\,rps vs.\ 240 offered at $\geq$99.99\% success, P99 589/476/571\,ms. Each 300\,s run completed 3--5 decisions because metric scrapes from the newly joined node timed out and delayed each observation; this lowers the decision rate and leaves the control logic unchanged from the single-node runs.

\niparagraph{Control-plane concurrency.} We ran $N$ concurrent real controller instances (live metrics, real GPT-5.2 decisions, real validator; observers record decisions without actuating) at 15\,s cadence, 240\,s per phase, on NLP Analysis and Text Generation.

\begin{table}[h]
\centering
\caption{Control-plane scale-out.}
\label{tab:concurrency}
\small
\begin{tabular}{@{}ccccc@{}}
\toprule
$N$ & Decisions & Failures & Mean latency (s) & P95 (s) \\
\midrule
2 & 32 & 0 & 1.03 & 4.13 \\
8 & 128 & 0 & 0.59 & 0.69 \\
16 & 256 & 0 & 0.64 & 0.71 \\
\bottomrule
\end{tabular}
\end{table}

Per-decision latency is flat in $N$; the $N{=}2$ P95 outlier (4.13\,s) is a first-call LLM warm-up. We observed zero failures across 416 decisions; cost scales linearly ($N \times$ \$0.01--0.05 per pipeline-hour).

% \newpage
% \input{checklist.tex}

\end{document}